\documentclass[]{colt2024}
\jmlrvolume{}
\jmlryear{2024}
\jmlrproceedings{}{}

\title[Short Title]{Unified theory of upper confidence bound policies for bandit problems targeting total reward, maximal reward, and more}
\usepackage{times}

\coltauthor{%
\Name{Nobuaki Kikkawa} \Email{kikkawa@mosk.tytlabs.co.jp}\\
\Name{Hiroshi Ohno} \Email{oono-h@mosk.tytlabs.co.jp}\\
\addr Toyota Central R\&D Labs., Inc., 41-1, Yokomichi, Nagakute, Aichi 480-1192, Japan}

\usepackage{algorithmic}
\usepackage{algorithm}

\DeclareMathOperator*{\argmax}{argmax}
\DeclareMathOperator*{\KL}{KL}
\DeclareMathOperator*{\erfc}{erfc}
\DeclareMathOperator*{\ierfc}{ierfc}

\begin{document}

\maketitle

\begin{abstract}%
The upper confidence bound (UCB) policy is recognized as an order-optimal solution for the classical total-reward bandit problem. While similar UCB-based approaches have been applied to the max bandit problem, which aims to maximize the cumulative maximal reward, their order optimality remains unclear. In this study, we clarify the unified conditions under which the UCB policy achieves the order optimality in both total-reward and max bandit problems. A key concept of our theory is the oracle quantity, which identifies the best arm by its highest value. This allows a unified definition of the UCB policy as pulling the arm with the highest UCB of the oracle quantity. Additionally, under this setting, optimality analysis can be conducted by replacing traditional regret with the number of failures as a core measure. One consequence of our analysis is that the confidence interval of the oracle quantity must narrow appropriately as trials increase to ensure the order optimality of UCB policies. From this consequence, we prove that the previously proposed MaxSearch algorithm satisfies this condition and is an order-optimal policy for the max bandit problem. We also demonstrate that new bandit problems and their order-optimal UCB algorithms can be systematically derived by providing the appropriate oracle quantity and its confidence interval. Building on this, we propose PIUCB algorithms, which aim to pull the arm with the highest probability of improvement (PI). These algorithms can be applied to the max bandit problem in practice and perform comparably or better than the MaxSearch algorithm in toy examples. This suggests that our theory has the potential to generate new policies tailored to specific oracle quantities.
\end{abstract}


\section{Introduction}

The bandit problem \citep{robbins1952some}, which aims to pull the best arm based on past observations, is a fundamental issue in decision making and machine learning.
Despite its simplicity, it serves as a theoretical foundation for real-world applications, such as Bayesian optimization \citep{srinivas2009gaussian}
and Monte Carlo Tree Search (MCTS) \citep{kocsis2006bandit,browne2012survey}.
The scope of the bandit problem extends beyond classical total reward maximization, encompassing a variety of problem settings.
These include settings that aim to maximize discounted total rewards \citep{garivier2011upper},
maximize a single best reward \citep{cicirello2005max}, or identify the optimal arm \citep{audibert2010best}.
The optimal policy also varies depending on differences in the reward distributions for each arm \citep{auer2002finite,bubeck2013bandits},
as well as whether the rewards are provided stochastically or adversarially \citep{auer2002finite,auer2002nonstochastic,bubeck2012regret}.
Many other policies have been proposed to address these varying conditions \citep{agrawal1995continuum,li2010contextual,yue2012k,besbes2014stochastic}.
Among these policies, the upper confidence bound (UCB) policy  \citep{auer2002finite,audibert2010best,garivier2011upper,kikkawa2024materials},
which utilizes confidence bounds of rewards, has proven its effectiveness under various conditions in stochastic settings.

Although UCB policies have proven generally effective in stochastic bandit problems, 
their validity is often evaluated case by case, depending on specific contexts and conditions.
As a result, when considering variant settings,
it becomes unclear what is permissible and what is not, 
based on existing knowledge derived from standard settings.
This lack of clarity poses a significant obstacle to applying UCB policies in the variant settings.
By establishing unified conditions for the effectiveness of UCB policies,
we aim to provide a theoretical guideline for the application of UCB policies across diverse problem settings.
As a step toward a unified theory,
this study provides a comprehensive proof of the effectiveness of the UCB policy
by focusing on two stochastic bandit problems:
the classic total-reward bandit problem, which aims to maximize cumulative total rewards \citep{robbins1952some},
and the max bandit problem (also known as the max $K$-armed bandit problem or extreme bandit problem),
which seeks to maximize the single best reward \citep{cicirello2005max,carpentier2014extreme,david2016pac}.
The key in this proof is the introduction of the oracle quantity $z_{k,t}$,
which addresses the differences in the objectives of the bandit problems.
This approach proves the effectiveness of the UCB policy
not only for the total-reward and max bandit problems but also for bandit problems maximizing other targets.
This flexibility of the target definition can improve practical exploration efficiency,
tailoring the policy to each target.
As an example, using the probability of improvement (PI) as the oracle quantity,
we constructed a PIUCB policy, which aims to maximize PI from the cumulative best reward.
This policy can effectively be used as one of the max bandit algorithms,
which performance is at least comparable to the MaxSearch algorithm,
a UCB algorithm proposed by \citet{kikkawa2024materials} for the max bandit problem.

This paper is organized as follows:
First, the Related Work section reviews previous research on the fundamentals of bandit problems and UCB policies.
Next, the Contributions section outlines the key achievements of this study.
The Preliminaries section introduces the formal definition of the bandit problems,
where we employ the oracle quantity $z_{k,t}$ to provide a unified representation of the different targets.
This approach offers a consistent framework for understanding various bandit problem settings.
In the Key Claims section and the Proofs section,
we present and prove a theorem and related propositions
that demonstrate the requirements for establishing order optimality of the UCB policy.
This theorem serves as a guide for developing UCB policies tailored to each target setting.
Following this, the Discussions section confirms that 
some total-reward UCB and MaxSearch algorithms are order-optimal using our theorem.
In addition, we construct an order-optimal PIUCB policy from our theorem using PI as the oracle quantity.
We also highlight the relation between our results and the previous research on regret lower bounds
\citep{lai1985asymptotically,burnetas1996optimal,garivier2019explore}.
Finally, in the Summary section, we conclude this paper
and discuss future prospects for applying our theorem across a wide range of scenarios.

\section{Related Work}

The classical total-reward bandit problem was first formulated by \citet{robbins1952some}.
About thirty years later, \citet{lai1985asymptotically} proved
that any consistent policy must pull suboptimal arms at the order of $\ln(t)$ in a stochastic setting.
This results established a lower bound on regret for the classical total-reward bandit problem.
This lower bound was later extended for general reward distributions by \citet{burnetas1996optimal}.
Furthermore, \citet{garivier2019explore} proposed a fundamental inequality to derive regret lower bounds in various settings.

UCB-based policies are well-known for achieving the regret that matches the order of the lower bound described above
in the classical bandit problems \citep{lai1985asymptotically,agrawal1995sample,auer2002finite}.
Notably, \citet{auer2002finite} proposed a simple policy, and since then,
the effectiveness of UCB policies has been proven individually
across various contexts and conditions \citep{audibert2010best,garivier2011upper,bubeck2013bandits},
For example, the UCB policies have been shown to be order-optimal
when the reward distributions have heavy tails \citep{bubeck2013bandits}.

While many bandit problems focus on the sum or expectation of rewards,
a distinct class of problems, known as the max bandit problems,
aims to maximize the single best reward \citep{cicirello2005max}.
In these problems, regret defined in a straightforward manner
approaches zero for most policies involving random selection \citep{nishihara2016no,kikkawa2024materials}.
This poses significant obstacles in proving order optimality of the algorithms.
However, \citet{kikkawa2024materials} recently demonstrated that UCB policies can be constructed
by focusing on the number of mistaken choices and employing a carefully designed oracle.
Their algorithm, called MaxSearch, practically outperforms other max bandit algorithms
\citep{streeter2006simple,achab2017max},
as well as the classical UCB \citep{auer2002finite} and the best arm identification algorithms \citep{audibert2010best}.
Nevertheless, the proofs for the regret lower bound and order optimality of the max bandit problem remain open problems.

In the context of Bayesian optimization, PI appears as an acquisition function
alongside expected improvement (EI) and UCB \citep{kushner1964new,jones1998efficient,srinivas2009gaussian}.
Although it is rare for PI to be used in the context of bandit problem,
there are instances of its use in early research on the max bandit problem \citep{streeter2006simple}.
One of the challenges in using PI in the bandit problem is ensuring its theoretical authenticity,
and our study justifies it by using PI as the oracle quantity.

\section{Contributions}

Our contributions are as follows:

\begin{itemize}
\item We present a unified representation of the bandit problem using the oracle quantity $z_{k,t}$ and the number of mistaken choices (hereafter referred to as failures) instead of regret.
\item We generalize the proof of the lower and upper bounds on failures for any consistent policy and the UCB policy, respectively, within the present framework.
\item We clarify the conditions under which the UCB policy becomes order-optimal in our framework, and providing the first proof of the order optimality of the MaxSearch algorithm for the max bandit problem.
\item We propose the PIUCB algorithm, which uses UCB of PI as the selection index, from our framework. This algorithm is at least comparable to the existing state-of-the-art max bandit algorithms, and in some cases, outperforms them.
\end{itemize}

\section{Preliminaries}

In our framework, bandit problems are generalized by the oracle quantity $z_{k,t}$ as the following definitions:

\begin{definition}[bandit problem]\label{definitionBanditProblem}
The $K$-armed bandit game is a game where a player repeatedly pulls one of the $K$ arms over $T$ times.
When a player pulls the $k$-th arm at the $t$-th round, they receive a reward $r_k(t)$
from an unknown reward distribution $f_k(r):=f(r;\theta_k)$,
where $\theta_k$ is a parameter set for the $k$-th arm’s distribution.
The sequence of a tuple of pulling arm indices and rewards over $T$ rounds
is denoted $\mathcal{R}_T:=\left\{\left(k(t),r_{k(t)}(t)\right)\right\}_{t\in[T]}$,
where $[T]:=\left\{1,2,\dots,T\right\}$.
The player's goal is to pull the optimal arm $k^*(t):=\argmax_{k\in[K]}z_{k,t}$ at round $t$,
where $z_{k,t}:=z_k\left(\left\{\theta_k\right\}_{k\in[K]},\mathcal{R}_T\right)$ is an oracle quantity.
The $K$-armed bandit problem, often referred to as the bandit problem,
is to find a policy $\pi:\left\{\mathcal{R}_{t-1},X_t\right\}\mapsto k(t)$ that pulls the optimal arm most frequently,
where $X_t$ is a sequence of random variables sampled from the uniform distribution $U(x)$,
used by the policy at each round $t$.
\end{definition}

\begin{remark}
If $z_{k,t}=\mathbb{E}\left[\sum_{t\in[T]}r_k(t)\right]=T\mu_k$
or simply $z_{k,t}=\mu_k$, where $\mu_k=\mathbb{E}\left[r_k(t)\right]$,
then Definition \ref{definitionBanditProblem} corresponds to the classical total-reward bandit problem.
\end{remark}

\begin{remark}
If $z_{k,t}=\mathbb{E}\left[\max_{t\in[T]}r_k(t)\right]$,
then Definition \ref{definitionBanditProblem} is corresponding to the max bandit problem in a sense of \citet{carpentier2014extreme}.
If $z_{k,t}$ is the EI of the $k$-th arm at time $t$ 
in terms of the single best reward from $r^\text{max}_{t-1}:=\max_{\tau\in[t-1]}r_{k(\tau)}(\tau)$,
then Definition \ref{definitionBanditProblem} corresponds to the max bandit problem in a sense of \citet{kikkawa2024materials}.
\end{remark}

\begin{definition}[failures]
A failure occurs when the player pulls the non-optimal arm $k(t)\neq k^*(t)$.
The number of failures, or simply called failures,
until round $t$ is defined as
\begin{equation}
N(t):=\sum_{\tau\in[t]}\mathbb{I}\left[k(\tau)\neq k^*(\tau)\right],
\end{equation}
where $\mathbb{I}[A]$ is the indicating function of the event $A$.
\end{definition}

\begin{remark}
The expected number of failures, $\mathbb{E}[N(T)]$,
after $T$ rounds is analogous to the expected regret, $R(T)$, in our framework.
This approach avoids difficulties in defining regret in the max bandit problem,
where $R(T)\to0$ at $T\to\infty$ for most policies, which causes fatal problems in theoretical analysis \citep{nishihara2016no,kikkawa2024materials}.
In the classical bandit settings, $R(T)$ and $\mathbb{E}[N(T)]$ are equivalent up to a constant factor
because $\varDelta_w \mathbb{E}[N(T)]\leq R(T)\leq\varDelta_s \mathbb{E}[N(T)]$,
where $\varDelta_k:=\mu_{k^*}-\mu_k$, and the arm indices $w$ and $s$ are the worst and sub-optimal arms, respectively.
\end{remark}

\begin{definition}[consistency]\label{definitionConsistency}
A policy is said to be “consistent” if the failures satisfies $\mathbb{E}[N(t)]=o(t^\delta)$, where $0<\delta<1$.
\end{definition}

\begin{remark}
Under a “consistent” policy, the number of successful pulls satisfies $\mathbb{E}[N^*(t)]=\varTheta(t)$,
where $N^*(t):=\sum_{\tau\in[t]}\mathbb{I}[k(\tau)=k^*(\tau)]$.
In other words, with infinite pulls, a consistent policy always pulls the best arm mostly.
Conversely, an inconsistent policy may fail indefinitely.
\end{remark}

\begin{definition}[UCB policy]\label{definitionUCBPolicy}
Assume that the oracle quantity $z_{k,t}$ satisfies
\begin{equation}\label{confidenceIntervalAssumption}
z_k^\text{LCB}\left(\mathcal{R}_{t-1}\right)\leq z_{k,t}\leq z_k^\text{UCB}\left(\mathcal{R}_{t-1}\right)
\end{equation}
with a confidence level $1-\alpha(t)$, where LCB is an abbreviation for lower confidence bound.
A policy that pulls $k(t)=\argmax_{k\in[K]}z_k^\text{UCB}\left(\mathcal{R}_{t-1}\right)$ at each round $t$
is called a UCB policy in general.
\end{definition}

\begin{remark}
In the classical bandit problem, the UCB of $z_{k,t}=\mu_k$
under the sub-gaussian assumption for $f(r;\theta_k)$ yields the well-known selection index \citep{auer2002finite},
\begin{equation}
z_k^\text{UCB}(\mathcal{R}_{t-1})=\bar{\mu}_k(\mathcal{R}_{t-1})+c\sqrt{\frac{-\ln\alpha(t)}{N_k(t)}},
\end{equation}
where $\bar{\mu}_k(\mathcal{R}_t):=[N_k(t)]^{-1}\sum_{\tau\in[t]}r_{k(\tau)}(\tau)\mathbb{I}[k=k(\tau)]$,
$N_k(t):=\sum_{\tau\in[t]}\mathbb{I}[k=k(\tau)]$,
and $c$ are the sample mean, number of the $k$-th arm pulling, and a positive constant, respectively.
\end{remark}
Our definitions are based solely on $z_{k,t}$. 
Consequently, the claims and algorithms in the following sections are concretely defined by specifying the form of $z_{k,t}$.

\section{Key Claims}
Under the above definitions, we claim the following.

\begin{theorem}[order optimality of UCB policy]\label{UCBTheorem}
The UCB policy gives $N(t)=\varTheta(\ln t)$ almost surely
for any oracle quantity $z_{k,t}$ that satisfies the following conditions.

\renewcommand{\labelenumi}{\alph{enumi})}
\begin{enumerate}
\item $z_{k,t}$ and $z_k^\text{UCB}(\mathcal{R}_{t-1})$ exist at each round $t$,
and the arms with their maxima are uniquely determined almost surely.
\item There exists $d>0$ such that the confidence interval
\begin{equation}\label{conditionForConfidenceInterval}
z_k^\text{UCB}(\mathcal{R}_t)-z_k^\text{LCB}(\mathcal{R}_t)
=O\left(\left[z_{k^*(t),t}-z_{k,t}\right]\beta_t^d\right)
\end{equation}
with $\alpha(t)=O(t^{-1})$, where $\beta_t:=(\ln t)/N_k(t)$.
\item There exists a modified parameter set $\left\{\theta_{k,k^*}'\right\}_{k,k^*\in[K]}$
such that
\begin{equation}
z_{k(\tau),\tau}':=z_{k(\tau)}\left(\left\{\theta_{k,k^*(\tau)}'\right\}_{k\in[K]},\mathcal{R}_\tau\right)>z_{k^*(\tau),\tau}
\end{equation}
and
\begin{equation}
f_{k(\tau)}(r)\ll f_{k(\tau)}'(r):=f\left(r;\theta_{k(\tau),k^*(\tau)}'\right)
\end{equation}
when $k(\tau)\neq k^*(\tau)$ for all $\tau\in[t]$,
where $f\ll g$ indicates that the distribution $f$ is absolutely continuous with respect to the distribution $g$.
\end{enumerate}
\end{theorem}

\begin{remark}
Theorem \ref{UCBTheorem} directly implies a generalized UCB algorithm as shown in Algorithm \ref{GUCB}.
If multiple arms accidentally have the same maximum $z_k^\text{UCB}(\mathcal{R})$ in any round, one of these arms will be pulled randomly.
\end{remark}
\begin{algorithm}
\caption{Generalized UCB}\label{GUCB}
\begin{algorithmic}[1]
\STATE $\mathcal{R}\leftarrow\varnothing$
\WHILE{$|\mathcal{R}|<T$}
\STATE $k\leftarrow\argmax_{k\in[K]}z_k^\text{UCB}(\mathcal{R})$ 
\STATE Pull the $k$-th arm and obtain reward $r$.
\STATE $\mathcal{R}\leftarrow\mathcal{R}\cup\left\{(k,r)\right\}$
\ENDWHILE
\end{algorithmic}
\end{algorithm}
The proof of Theorem \ref{UCBTheorem} is almost the same as the corresponding proofs for the classical bandit problem 
\citep{lai1985asymptotically,auer2002finite,bartlett2014learning,kikkawa2024materials}.
This proof consists of the proofs of the following two propositions.

\begin{proposition}[upper bound of failures]\label{PropositionUpperBound}
The UCB policy guarantees $N(t)=O(\ln t)$ for any oracle quantity that satisfies conditions (a) and (b) in Theorem \ref{UCBTheorem}.
\end{proposition}

\begin{proposition}[lower bound of failures]\label{propositionLowerBound}
For any consistent policy and oracle quantity, the failures $N(t)=\varOmega(\ln t)$ almost surely
if conditions (a) and (c) in Theorem \ref{UCBTheorem} are satisfied.
\end{proposition}

\section{Proofs}
\subsection{Proof for Proposition \ref{PropositionUpperBound}}
Assume the confidence bounds are given by Equation (\ref{confidenceIntervalAssumption}) in condition (b)
with a confidence level of $1-\alpha(t)$.
Then, the number of times any of the $z_{k,t}$ violates
these confidence bounds under the UCB policy until round $t$ is given by
\begin{equation}\label{inequality2}
\sum_{\tau\in[t]}\mathbb{I}\left[\bigcup_{k\in[K]}z_{k,\tau}<z_k^\text{LCB}\left(\mathcal{R}_{\tau-1}\right)\cup z_k^\text{UCB}\left(\mathcal{R}_{\tau-1}\right)<z_{k,\tau}\right]\leq K\sum_{\tau\in[t]}\alpha(\tau)=O(\ln t)
\end{equation}
under the assumption of $\alpha(t)=O(t^{-1})$ in condition (b).
In all other cases, where the confidence bounds of Equation (\ref{confidenceIntervalAssumption}) hold for all $k$,
the UCB policy pulls a non-optimal arm $k\neq k^*(t)$
when $z_k^\text{UCB}(\mathcal{R}_{t-1})\ge z_{k^*(t)}^\text{UCB}(\mathcal{R}_{t-1})$ at time $t$.
Then,
\begin{equation}\label{inequality1}
z_{k^*(t),t}\le z_{k^*(t)}^\text{UCB}(\mathcal{R}_{t-1})\le z_k^\text{UCB}(\mathcal{R}_{t-1})
\le z_k^\text{UCB}(\mathcal{R}_{t-1})+z_{k,t}-z_k^\text{LCB}(\mathcal{R}_{t-1})
\end{equation}
must hold in this case.
Thus, if Equation (\ref{conditionForConfidenceInterval}) holds,
inequality (\ref{inequality1}) implies $N_k(t)\leq O(\ln t)$.
Then, $\sum_{k\in[K]}N_k(t)=O(\ln t)$.
Therefore, combining it with Equation (\ref{inequality2}),
Proposition \ref{PropositionUpperBound} is established since all possible cases are covered.

\begin{remark}
The improvement of this proof over the previous studies \citep{auer2002finite,kikkawa2024materials} lies in focusing on the width of the confidence interval and identifying the conditions under which $N(t)=O(\ln t)$ holds, without providing an explicit expression for the confidence bounds.
\end{remark}

\subsection{Proof for Proposition \ref{propositionLowerBound}}

From Definition \ref{definitionBanditProblem},
the probability of an event $N(t)\leq a_t$ at time $t$
under the original and modified distributions in condition (c) can be written as:
\begin{equation}\label{originalProbability}
\mathbb{P}[N(t)\leq a_t]=\int_{N(t)\leq a_t} \prod_{\tau\in[t]}f_{k(\tau)}(r_\tau)dr_\tau \prod_{x\in X_\tau}U(x)dx,
\end{equation}
\begin{equation}\label{modifiedProbability}
\mathbb{P}'[N(t)\leq a_t]=\int_{N(t)\leq a_t}\prod_{\tau\in[t]}f_{k(\tau)}'(r_\tau)dr_\tau \prod_{x\in X_\tau}U(x)dx.
\end{equation}
for any $z_{k,\tau}$ which can uniquely determine $k^*(\tau)$.
These probabilities are related to each other, as shown in the following lemma.

\begin{lemma}\label{relationPPPrime}
Assume $f_{k(\tau)}(r)\ll f_{k(\tau)}'(r)$ for all $\tau\in[t]$. Then,
\begin{equation}
\mathbb{P}[N(t)\leq a_t]\leq e^{N(t)M}\mathbb{P}'[N(t)\leq a_t],
\end{equation}
holds almost surely as $t\to\infty$ when setting $\theta_{k^*,k^*}'=\theta_{k^*}$.
The factor $M$ is given as
\begin{equation}\label{definitionM}
M:=\lim_{t\to\infty}\left[\frac{1}{N(t)}\sum_{k,k^*\in[K]}N_{k,k^*}(t)\KL(f_k,f_k')\right]=O(1)
\end{equation}
for $N(t)$, where $N_{k,k^*}(t):=\sum_{\tau\in[t]}\mathbb{I}[k=k(t)\cap k^*=k^*(t)]$
and $\KL(f,g)$ is the Kullback-Leibler (KL) divergence of $f$ from $g$.
\end{lemma}
\begin{proof}
From the definitions
\begin{align}
\mathbb{P}[N(t)\leq a_t]
&=\int_{N(t)\leq a_t} e^{L_t}\prod_{\tau\in[t]}f_{k(\tau)}'(r_\tau)dr_\tau\prod_{x\in X_\tau}U(x)dx \\
&=\left[\int_{N(t)\leq a_t\cap L_t\leq b_t}+\int_{N(t)\leq a_t\cap L_t>b_t}\right]
e^{L_t}\prod_{\tau\in[t]}f_{k(\tau)}'(r_\tau)dr_\tau\prod_{x\in X_\tau}U(x)dx \\
&\leq e^{b_t} \mathbb{P}'[N(t)\leq a_t]+\mathbb{P}[N(t)\leq a_t\cap L_t>b_t] \label{inequality3}
\end{align}
for any $b_t$, where $L_t:=\sum_{k,k^*\in[K]}
\sum_{n\in[N_{k,k^*}(t)]}\ln[f_k(r_n)/f_k'(r_n)]$.
Here,
\begin{equation}\label{lawOfLargeNumbers}
\frac{1}{N_{k,k^*}(t)}\sum_{n\in[N_{k,k^*}(t)]}\ln\frac{f_k(r_n)}{f_k'(r_n)}\to\KL(f_k,f_k')<\infty
\end{equation}
as $t\to\infty$, as a result of the law of large numbers under the assumption of the absolute continuity.
Thus, $L_t\leq N(t)M$ holds almost surely as $t\to\infty$,
since $M=M\left(\left\{\theta_{k,k^*}'\right\}_{k,k^*\in[K]}\right)$ satisfies Equation (\ref{definitionM}).
Because $N(t)=\sum_{k,k^*\in[K]\cap k\neq k^*}N_{k,k^*}(t)$
and $\KL(f_{k^*},f_{k^*}')=0$ when $\theta_{k,k^*}'=\theta_{k^*}$, $M=O(1)$ for $N(t)$.

When $L_t\leq N(t)M$ holds, the second term of inequality (\ref{inequality3}) becomes,
\begin{equation}\label{probability1}
\mathbb{P}[N(t)\leq a_t\cap L_t>b_t]=\mathbb{P}[N(t)\leq a_t\cap L_t>b_t\cap L_t\leq N(t)M].
\end{equation}
Considering the event of probability (\ref{probability1}), $b_t<N(t)M$ should be satisfied.
Then, this probability vanishes to $0$ if $b_t\ge N(t)M$.
Applying this result to inequality (\ref{inequality3}), Lemma \ref{relationPPPrime} is obtained.
\end{proof}
Using Lemma \ref{relationPPPrime} and the Markov inequality,
\begin{align}
\mathbb{P}[N(t)\leq a_t]
&\leq e^{N(t)M}\mathbb{P}'[N(t)\leq a_t]=e^{N(t)M}\mathbb{P}'[t-N(t)\ge t-a_t] \notag \\
&\leq e^{N(t)M}\frac{\mathbb{E}'[t-N(t)]}{t-a_t} \label{inequality4}
\end{align}
holds almost surely,
where $\mathbb{E}'[y]$ is the expectation of random variable $y$ under $\mathbb{P}'$.
Here, from Definition \ref{definitionConsistency}, $\mathbb{E}'[t-N(t)]=o(t^\delta)$ for any consistent policy
if $\left\{\theta_{k(t),k^*(t)}'\right\}_{k,k^*\in[K]}$ satisfies $z_{k(\tau),\tau}'>z_{k^*(\tau),\tau}$
when $k(\tau)\neq k^*(\tau)$ for each $\tau\in[t]$.
Then, the probability in Equation (\ref{inequality4}) vanishes to $0$ as $t\to\infty$ when $a_t\ge (1-\delta)(\ln t)/M$.
In other words, $N(t)>(1-\delta)(\ln t)/M$ holds almost surely as $t\to\infty$.
Considering the worst case of $\delta$, we obtain $N(t)>(\ln t)/M$.
For this lower bound, a smaller $M$ should be taken for faster convergence of $\mathbb{P}[N(t)\leq a_t]$
under $z_{k(\tau),\tau}'>z_{k^*(\tau),\tau}$ for each $\tau\in[t]$.
Then, we obtain the following lemma.

\begin{lemma}\label{lemmaMTilde}
$N(t)>(\ln t)/\tilde{M}$ holds almost surely as $N(t)\to\infty$ for any consistent policy and oracle quantity,
where
\begin{equation}
\tilde{M}:=\inf
\left\{M\left(\left\{\theta_{k,k^*}'\right\}_{k,k^*\in[K]}\right)
\middle|z_{k(\tau),\tau}'>z_{k^*(\tau),\tau},f_{k(\tau)}(r)\ll f_{k(\tau)}'(r)\right\}.
\end{equation}
\end{lemma}
Considering the existence of $\tilde{M}$, we obtain Proposition \ref{propositionLowerBound}.

\begin{remark}
This proof is almost similar to the previous proofs \citep{lai1985asymptotically,bartlett2014learning}.
Our main claim is the difference of $z_{k,t}$ is absorbed into the definition of $N(t)$ 
and only appears in the search area of the infimum of $\tilde{M}$ in general. 
Note that the modified parameter set, $\theta_{k,k^*}$, has $K^2$ elements
due to the time dependency of $k^*$.
\end{remark}

\section{Discussion}

Theorem \ref{UCBTheorem} implies that order-optimal UCB algorithms for maximizing $z_{k,t}$
can be created only by providing $z_{k,t}$, $z_k^\text{UCB}(\mathcal{R}_t)$, 
and the function shape of $f_k(r)$ to satisfy the three conditions in this theorem.
Especially, if the support of $f(r;\theta)$ is the same regardless of $\theta$, 
the absolute continuity in condition (c) holds.
Also, if $z_{k,t}$ is continuous on an open interval,
the uniqueness of $z_{k,t}$ in condition (a)
and existence of $z_{k(\tau),\tau}'>z_{k^*(\tau),\tau}$ in condition (c) holds almost surely.
Then, we obtain the following corollary which indicates the conditions
that the confidence interval must satisfy for the UCB policy to be order-optimal:

\begin{corollary}\label{simpleCorollary}
Assume the support of $f(r;\theta)$ is the same regardless of $\theta$.
Also, assume $z_{k,t}$ is continuous on an open interval.
Then, the UCB policy ensures $N(t)=\varTheta(\ln t)$ almost surely
if the $k(t):=\argmax_{k\in[K]}z_k^\text{UCB}(\mathcal{R}_{t-1})$
is uniquely determined almost surely and there exists $d>0$
such that equation (\ref{conditionForConfidenceInterval}) holds with  $\alpha(t)=O(t^{-1})$.
\end{corollary}
Using this corollary, we discuss the order optimality of existing and newly proposed bandit problems and their UCB policies as follows.

\subsection{Order optimality of existing UCB policies for classical bandit problem}

The well-known UCB policy for the classical bandit problem \citep{auer2002finite} is based on the fact
that the confidence interval for the mean $\mu$ with $\alpha(t)=O(t^{-1})$ is written as
\begin{equation}
\bar{\mu}_k(\mathcal{R}_t)-c\sqrt{\beta_t}<\mu_k<\bar{\mu}_k(\mathcal{R}_t)+c\sqrt{\beta_t},
\end{equation}
where $\bar{\mu}_k(\mathcal{R}_t)$ is the estimated mean from the $N_k(t)$ rewards obtained from the arm $k$,
and $c$ is a positive constant that determines the balance of exploration and exploitation.
In this case, the width of the interval, $2c\sqrt{\beta_t}$ is $O\left([\mu_{k^*}-\mu_k][\beta]^{1/2}\right)$.
Thus, from Corollary \ref{simpleCorollary} with $d=1/2$, the UCB policy is order-optimal.

The UCB policy for heavy-tailed distributions \citep{bubeck2013bandits} is based on the confidence interval as
\begin{equation}
\bar{\mu}_k(\mathcal{R}_t)-c\beta_t^\frac{\varepsilon}{1+\varepsilon}
<\mu_k<\bar{\mu}_k(\mathcal{R}_t)+c\beta_t^\frac{\varepsilon}{1+\varepsilon}
\end{equation}
 with $\alpha(t)=O(t^{-1})$, where $\varepsilon$ is a parameter related to the tail distribution.
The width of the interval is $2c\beta_t^{\varepsilon/(1+\varepsilon)}$.
Thus, from Corollary \ref{simpleCorollary} with $d=\varepsilon/(1+\varepsilon)$,
the UCB policy is order-optimal.

\subsection{Order optimality of a UCB policy for max bandit problem}

The MaxSearch algorithms were recently proposed by \citet{kikkawa2024materials} as a UCB policy for the max bandit problem.
Their approach corresponds to the case
where $z_{k,t}=\int_{r_{t-1}^\text{max}}^\infty S_k(u)du$ in our framework,
where $S_k(r)$ is the survival function of $f_k(r)$.
Under the assumption of the Gaussian reward distributions,
they show that $\sqrt{2}z_{k,t}=
\sigma_k\ierfc(\zeta_{k,t})$,
where $\ierfc(x):=\int_x^\infty\erfc(x)dx=\exp(-x^2)/\sqrt{\pi}-x\erfc(x)$,
$\erfc(x)$ is the complementary error function,
and $\zeta_{k,t}:=(r_{t-1}^\text{max}-\mu_k)/\sqrt{2}\sigma_k$.
Then, the confidence interval of $z_{k,t}$ is written as:
\begin{equation}\label{EIConfidenceInterval}
\hat{\sigma}_{k,-}\ierfc\frac{r_{t-1}^\text{max}-\hat{\mu}_{k,-}}{\sqrt{2}\hat{\sigma}_{k,-}}
<\sqrt{2}z_{k,t}<\hat{\sigma}_{k,+}\ierfc\frac{r_{t-1}^\text{max}-\hat{\mu}_{k,+}}{\sqrt{2}\hat{\sigma}_{k,+}},
\end{equation}
where $\hat{\mu}_{k,-}<\mu_k<\hat{\mu}_{k,+}$ and $\hat{\sigma}_{k,-}^2<\sigma_k^2<\hat{\sigma}_{k,+}^2$
are the confidence interval of $\mu_k$ and $\sigma_k^2$ of the Gaussian distributions, respectively.
Therefore, the width of the interval becomes
$O(\exp(-\zeta^2_{s,t})\beta_t^{1/2})$ with $\alpha(t)=O(t^{-1})$,
which satisfies the condition in Corollary \ref{simpleCorollary} with $d=1/2$.

\subsection{Bandit problem to select the arm with the best probability of improvement}

As the last example of the use of Corollary \ref{simpleCorollary},
we consider the case with $z_{k,t}=S(r_{t-1}^\text{max})$.
This case corresponds to using PI from $r_{t-1}^\text{max}$
as the oracle quantity and UCB of PI as the selection index.
The concept of UCB of PI can be found in the early study
on the max bandit problem by \citet{streeter2006simple},
although their theoretical foundation remains incomplete.
By Corollary \ref{simpleCorollary},
the use of the UCB of PI is justified as the order-optimal approach
for a bandit problem to pull the arm with the best PI at each $t$.
Thus, we call this new bandit problem the PI bandit problem,
and the order-optimal policy as PIUCB policy.

Under the assumption of the Gaussian reward distributions,
the oracle quantity can be simply written as $z_{k,t}=\erfc(\zeta_{k,t})/2$.
Because $\erfc(x)$ is a monotonically decreasing function,
an oracle quantity $\tilde{z}_{k,t}=-\zeta_{k,t}$
also yields the same $k^*(t)$ of the original $z_{k,t}$ in the PI bandit problem.
The confidence bounds of $\tilde{z}_{k,t}$ is written by
$(\hat{\mu}_{k,-}-r_{t-1}^\text{max})/\hat{\sigma}_{k,-}
<\tilde{z}_{k,t}<(\hat{\mu}_{k,+}-r_{t-1}^\text{max})/\hat{\sigma}_{k,+}$.
Given the confidence levels $1-\alpha_\mu(t)$
and $1-\alpha_\sigma(t)$ for $\mu_k$ and $\sigma_k^2$, respectively,
these intervals can be directly calculated
using the $p$-th quantiles $t_{p,n}$ and $\chi^2_{p,n}$ of the student $t$ and $\chi^2$ distributions with $n$ degrees of freedom.
Then, we derive the PIUCB algorithm as follows:

\begin{algorithm}
\caption{PIUCB}\label{PIUCB}
\begin{algorithmic}[1]
\REQUIRE $\alpha_\mu(t)$ and $\alpha_\sigma(t)$
\STATE $\mathcal{R}\leftarrow\varnothing$
\WHILE{$|\mathcal{R}|<T$}
\STATE  Calculate $\bar{\mu}_k$, $\bar{\sigma}_k$, $N_k$, and $r^\text{max}_{t-1}$ from $\mathcal{R}$
\STATE  $\hat{\mu}_k\leftarrow\bar{\mu}_k+\bar{\sigma}_kt_{[1-\alpha_\mu(t)]/2,N_k-1}/\sqrt{N_k}$
\STATE  $\hat{\sigma}_k^2\leftarrow (n-1)\sigma_k^2/\chi_{\alpha_\sigma(t)/2,N_k-1}^2$
\STATE  $\hat{z}_k\leftarrow(\hat{\mu}_k-r^\text{max}_{t-1})/\hat{\sigma}_k$
\STATE $k\leftarrow\argmax_{k\in[K]}\hat{z}_k$ 
\STATE Pull the $k$-th arm and obtain reward $r$.
\STATE $\mathcal{R}\leftarrow\mathcal{R}\cup\left\{(k,r)\right\}$
\ENDWHILE
\end{algorithmic}
\end{algorithm}
Using the asymptotic approximations $\hat{\mu}_{k,\pm}\to\bar{\mu}_k\pm c_\mu \bar{\sigma}_k\sqrt{\beta_t}$
and $\hat{\sigma}_{k,\pm}\to\bar{\sigma}_k (1\pm c_\sigma \sqrt{\beta_t/8})$
for large $t$ at $\alpha_\mu(t)=t^{-c_\mu^2/2}$ and $\alpha(t)=t^{-c_\sigma^2/2}$ [4],
the confidence interval of $\tilde{z}_{k,t}$ is $O(r_{t-1}^\text{max}\beta_t^{1/2})$.
Then, in the case with $c_\mu,c_\sigma>\sqrt{2}$ and $c_\sigma^2\neq8c_\mu^2$,
the PIUCB policy satisfies the conditions in Corollary \ref{simpleCorollary} with $d=1/2$,
and this policy becomes order-optimal.
In the special case, $c_\sigma^2=8c_\mu^2$,
the confidence interval of $\tilde{z}_{k,t}$ becomes $O(r_{t-1}^\text{max}\beta_t)$
because of the cancellation of $\sqrt{\beta_t}$ terms.
In this case, the conditions in Corollary \ref{simpleCorollary} are satisfied with $d=1$.
Accordingly, the UCB policy which pulls the arm with $z_{k,t}^\text{UCB}$
gives $N(t)=\varTheta(\ln t)$ almost surely for $c_\mu,c_\sigma>\sqrt{2}$.

\begin{remark}
The max bandit problem and the PI bandit problem are similar but not completely identical.
For example, in the case of the two-armed bandit problem with the Gaussian reward distributions
with $(\mu_1,\sigma_1,\mu_2,\sigma_2,r_{t-1}^\text{max})=(0,1,-1,2,0.5)$,
the max bandit problem constructed by $\sqrt{2}z_{k,t}=\sigma_k\ierfc(\zeta_{k,t})$
gives $k^*(t)=2$ ($z_{1,t}\approx0.198<z_{2,t}\approx0.262$),
but the PI bandit problem constructed by $\tilde{z}_{k,t}=-\zeta_{k,t}$
gives $k^*(t)=1$ ($z_{1,t}\approx0.309>z_{2,t}\approx0.227$).
Even so, in the case with $r_{t-1}^\text{max}>1$,
both bandit problems give the same best arm, $k^*(t)=2$.
Accordingly, the PIUCB policy can be used for the max bandit problem
as well as the MaxSearch policy in the case when the reward distributions follow the Gaussian.
Because the PIUCB policy gives a simpler selection index than that of the MaxSearch policy,
it is valuable both in terms of implementation and theoretical understanding.
\end{remark}

\begin{remark}
The $\zeta_{k,t}$ is propotionally equivalent to the z-score, or standard score of $r_{t-1}^\text{max}$.
Then, the PIUCB algorithms aim to select the arm with a low z-score.
The arm with a low z-score has greater potential to improve the z-score than that with a high z-score.
Because the improvement of z-score corresponds to the improvement of $r_{t-1}^\text{max}$,
the PIUCB algorithms work well.
\end{remark}

\subsection{Relation to previous theory on lower bounds}

Lemma \ref{lemmaMTilde} is related to the previous studies for the regret lower bounds
in the classical total-reward bandit problem \citep{lai1985asymptotically,burnetas1996optimal,garivier2019explore}.
In our framework, the classical problem is characterized by the time-independent oracle quantity,
$z_{k,t}=z_k:=\mu_k$ and $k^*(t)=k^*:=\max_{k\in[K]}\mu_k$.
Then, the condition $z_{k(\tau),\tau}'>z_{k^*(\tau),\tau}$ in Lemma \ref{lemmaMTilde} simply corresponds to $\mu_k'>\mu_{k^*}$.

Assuming the Bernoulli reward distribution $f(r;\mu_k)$,
the magnitude of the KL divergence corresponds to the difference in $\mu_k$.
Then, we obtain $\KL(f_k,f_{k}')>\KL(f_k,f_{k^*})$ and $\KL(f_k,f_{k^*})>\KL(f_s,f_{k^*})$ when $k\neq k^*$.
Therefore, $\tilde{M}>\KL(f_s,f_{k^*})$ and $N(t)>(\ln t)/\KL(f_s,f_{k^*})$ hold almost surely,
which corresponds to the results of \citet{lai1985asymptotically}.
On the other hand, assuming arbitrary reward distributions,
the relation between the KL divergence and the mean remines uncertain.
Because of this, we can only state Lemma \ref{lemmaMTilde},
which corresponds to the results of \citet{burnetas1996optimal}.

By considering the law of large numbers for $N_k (t)$ after equation (\ref{lawOfLargeNumbers}),
\begin{equation}
b_t\ge\lim_{t\to\infty}\sum_{k,k^*\in[K]\cap k\neq k^*}\mathbb{E}[N_{k,k^*}(t)]\KL(f_k,f_k')
\ge kl(E[Z],E'[Z])
\end{equation}
is obtained instead of $b_t\ge N(t)M$,
where $kl(\mu,\mu')$ is the KL divergence of the Bernoulli distributions.
The last inequality corresponds to the fundamental inequality proposed by \citet{garivier2019explore}.
It is known that the results of \citet{lai1985asymptotically} or \citet{burnetas1996optimal}
can be obtained from the fundamental inequality by setting $Z$ to $N_s(t)/T$
in the classical bandit settings (See the original article \citep{garivier2019explore} for detail).

\section{Experiments}
\subsection{Performance on synthetic problems}
We examine the performance of PIUCB policies using three max bandit problems proposed by \citep{kikkawa2024materials}.
In the “easy” problem, one should pull arm 3 mostly among the three arms
with the Gaussian reward distributions of $(\mu_1,\sigma_1,\mu_2,\sigma_2,\mu_3,\sigma_3)=(1,1,0,2,-1,3)$.
The classical bandit algorithms will wrongly pull arm 1 because this arm has the best $\mu$,
although it has the smallest standard deviation.
The “difficult” problem is also constructed by the three Gaussian reward distributions
with $(\mu_1,\sigma_1,\mu_2,\sigma_2,\mu_3,\sigma_3)=(-0.2,1.1,0,1,-0.8,1.2)$.
If the total number of pulls $T$ satisfies $10^2\ll T\ll 10^9$,
the arm 1 with the medium mean and standard deviation is the best arm for the max bandit problem.
In the “unfavorable” problem constructed by the Gaussian reward distributions
with $(\mu_1,\sigma_1,\mu_2,\sigma_2,\mu_3,\sigma_3)=(1,1,0,1,-1,1)$,
the classical bandit algorithms can find the best arm 1 faster than the max bandit algorithms
because the best arm is determined only by the mean.
The max bandit algorithms estimate the standard deviations as well as the mean,
which requires more effort than estimating only the mean.

Using the three bandit problems, we compare the classical UCB algorithm,
the MaxSearch algorithm, and two PIUCB algorithms (Figure \ref{synthetic}).
We used the same implementations of UCB and MaxSearch[Gaussian] as the previous work [4] in our experiments.
We also used the PIUCB algorithms with $c_\mu=c_\sigma=\sqrt{2}$ and $(c_\mu,c_\sigma)=(1/2,\sqrt{2})$,
which correspond to the cases with $d=1/2$ and $d=1$ in Corollary \ref{simpleCorollary}, respectively.
The former and latter PIUCB algorithms are denoted as PIUCB1 and PIUCB2 in this paper.
Although PIUCB2 uses smaller $c_\mu$ than the theoretical lower bounds,
the results in the following paragraph suggest that PIUCB2 empirically outperforms MaxSearch and PIUCB1.

In the easy problem, the transitions of cumulative maximum, $r_t^\text{max}$,
of MaxSearch, PIUCB1, and PIUSB2 in Figure \ref{synthetic}a (left) exhibit nearly identical trends
and demonstrate better performance than the classical UCB algorithm.
The logarithmic plots of MaxSearch and PIUCBs for the numbers of failures, $N(t)$,
in a sense of the max bandit problem, shown in Figure \ref{synthetic}a (right),
generally show the linears behavior against $\ln t$ as expected by Corollary \ref{simpleCorollary}.
As expected, the failures of the classical UCB algorithm do not become linear.
The absolute counts of $N(t)$ for MaxSearch and PIUCB2 are significantly lower than those for PIUCB1,
which means that MaxSearch and PIUCB2 perform better than PIUCB1 for the easy problem.
In the results from the difficult problem shown in Figure \ref{synthetic}b,
MaxSearch and PIUCBs also outperform the classical UCB algorithm in the exploration for high $r_t^\text{max}$.
The failures of MaxSearch and PIUCBs are also smaller than the classical UCB algorithm.
The theoretically predicted linearity of $N(t)$ against $\ln t$ cannot be definitively confirmed from these experiments
because of the limitation of the number of trials.
The PIUCB2 algorithm demonstrates significantly higher performance than those of MaxSearch and PIUCB1 at $t\approx10,000$.
In the results from the unfavorable problem in Figure \ref{synthetic}c,
all algorithms have no significant difference in $r_t^\text{max}$.
However, according to the results of the number of failures,
the classical UCB fails less than $50$ times in the $10,000$ trials.
This result suggests that the classical UCB algorithm is suitable for the unfavorable problem as theoretically expected.
Among the UCB algorithms for the max bandit problems,
PIUCB2 outperforms MaxSearch and PIUCB1 at $t\gtrsim100$ in terms of failures.
Considering the outcomes of the three bandit problems,
PIUCB1 performs comparably to MaxSearch, and PIUCB2 performs as well as or better than MaxSearch.

\begin{figure}[htbp]
  \centering
  \includegraphics[width=\linewidth]{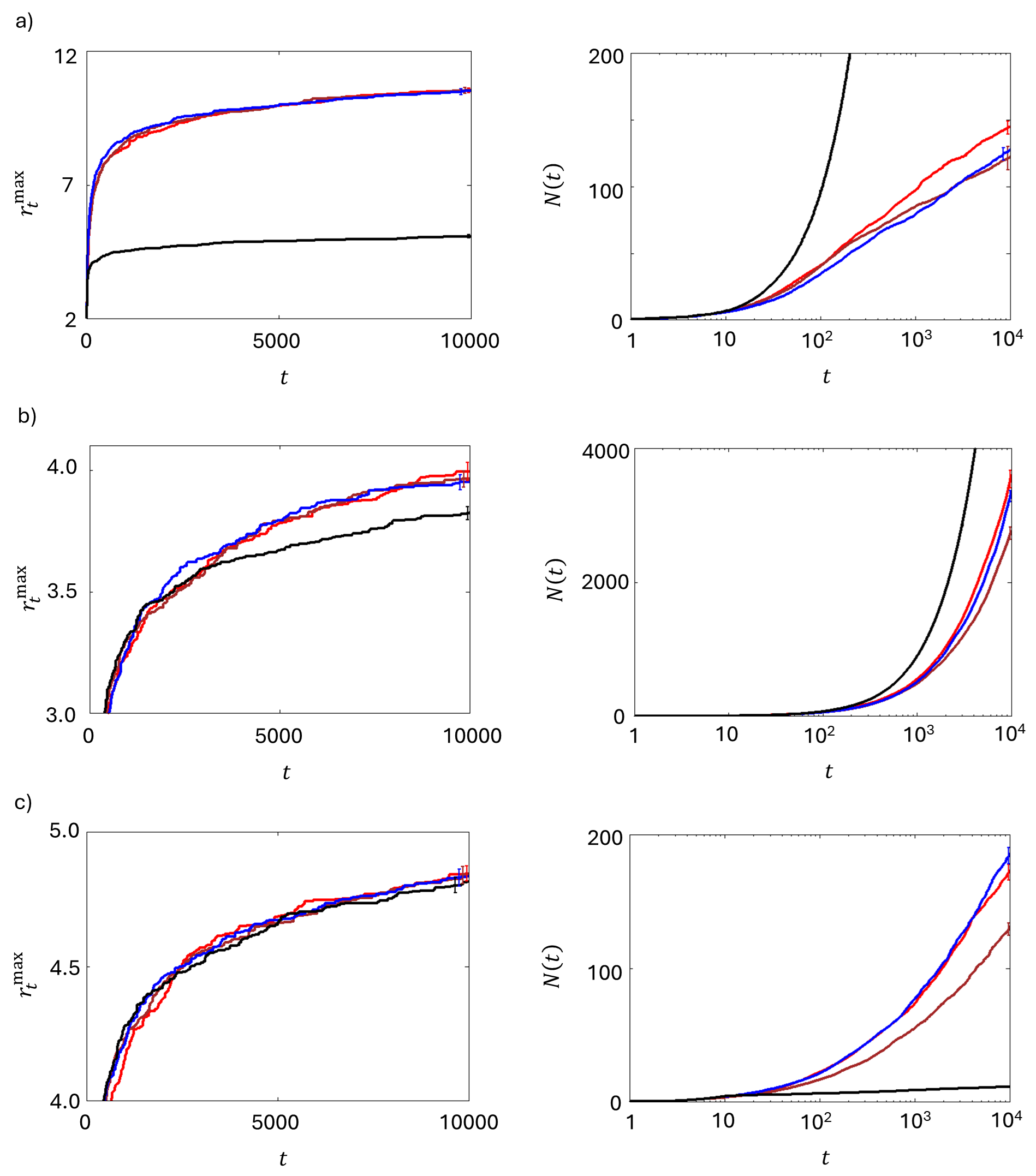}
  \caption{
  Transition plots of the cumulative maximum rewards (left),
  and the numbers of failures in a sense of the max bandit problem (right) for the three synthetic problems.
  The means and standard errors of each $100$ independent runs are shown in the graphs using colored lines and error bars.
  (a) the easy problem, (b) the difficult problem, (c) the unfavorable problem.
  The red, brown, blue, and black indicate PIUCB1, PIUCB2, MaxSearch, and classical UCB, respectively.
  }
  \label{synthetic}
\end{figure}

\subsection{Performance on material exploration}
We also examine the performance of the material exploration problems \citep{kikkawa2024materials}.
Each problem explores molecules with high physical property $A$ using MCTS \citep{kocsis2006bandit,browne2012survey}.
As the property $A$, the proposed problems use boiling temperature, $T_b$ [K],
critical pressure, $P_c$ [bar], viscosity, $\eta$ [Pa$\cdot$s],
or topological polar surface area, $A_\text{TPSA}$ [$\text{\AA}^2$] ($\text{\AA}=0.1$ nm),
calculated by empirical equations \citep{joback1987estimation,landrum2006rdkit}
of molecular structure described by SMILES strings \citep{weininger1988smiles}.
The SMILES strings are generated by the selections of the fragment strings
ruled by the context-free grammar \citep{hopcroft2001introduction} with the following substitution rules:
\begin{equation}
\begin{aligned}
S&\leftarrow \text{C}(X)(Y)(Y)(Y), \text{C}(\text{=O})(Y)(Y),
\text{C}(Y)(Y)(\text{=C}(Y)(Y)), \text{or } \text{C}(\text{=O})(\text{O}(Y))(Y) \\
X&\leftarrow[\text{H}], \text{F}, \text{Cl}, \text{Br}, \text{C}(X)(Y)(Y), \text{O}(Y),
\text{N}(Y)(Y), \text{C}(\text{=O})(Y), \text{C}(Y)(\text{=C}(Y)(Y)), \\
&\text{or } \text{C}(\text{=O})(\text{O}(Y)) \\
Y&\leftarrow [\text{H}], \text{F}, \text{Cl}, \text{Br}, \text{C}(X)(Y)(Y), \text{C}(\text{=O})(Y),
\text{C}(Y)(\text{=C}(Y)(Y)), \text{or } \text{C}(\text{=O})(\text{O}(Y)),
\end{aligned}
\end{equation}
where $S$ is the start strings. One repeatedly selects these substitutions until all $S$, $X$, and $Y$ vanish.
Then, one obtains a complete SMILES string, corresponding to a molecule structure.
Because all selections in this molecular generation process can be expressed by a game tree,
we can explore molecules using MCTS with one generation as one trial.
Employing the selection indices of bandit algorithms for each selection in MCTS,
we can examine the performance of the proposed selection indices for material exploration.
The original study pointed out that the $T_b$ maximizing problem
corresponds to the unfavorable problem described in the previous section,
where the UCB algorithm outperforms the MaxSearch algorithm.
In addition, the $\eta$ maximizing problem has enormously high values in the search space,
and the $A_\text{TPSA}$ maximizing problem probably has a bound near the value that can be easily discovered.
These differences allow us to verify the practical performance for non-normally distributed reward distributions.
Please see the original paper \citep{kikkawa2024materials} for details, especially for implementations of molecular generations and MCTS.

Figure \ref{material} shows the results of PIUCBs for exploring the molecules with high physical properties.
The results of MaxSearch obtained by \citet{kikkawa2024materials} are also shown in the graphs.
In the $T_b$ maximizing problem (Figure \ref{material}a),
PIUCB2 significantly outperformes both PIUCB1 and MaxSearch.
This is natural to consider the performance improvement of PIUCB2 for the unfavorable problem in Figure \ref{synthetic}c
and the correspondence between the $T_b$ maximizing and the unfavorable problem.
In the results of the $P_c$ maximizing problem in Figure \ref{material}b,
the PIUCB algorithms clearly demonstrate higher performance than MaxSearch.
Contrary to this, in the results of the $\eta$ maximizing problem shown in Figure \ref{material}c,
the MaxSearh algorithm probably outperforms PIUCBs,
although the significance cannot be concluded due to the large confidence interval.
The previous study discussed that the cause of this large interval
is the existence of extremely higher rewards in the $\eta$ maximizing problem.
In the $A_\text{TPSA}$ maximizing problem (Figure \ref{material}d),
PIUCBs shows higher performance near $T=3,000$ than MaxSearch.
However, the performance is reversed at $T=10,000$.
These results suggest that PIUCBs are comparable to MaxSearch and better than in some cases.

\begin{figure}[htbp]
  \centering
  \includegraphics[width=\linewidth]{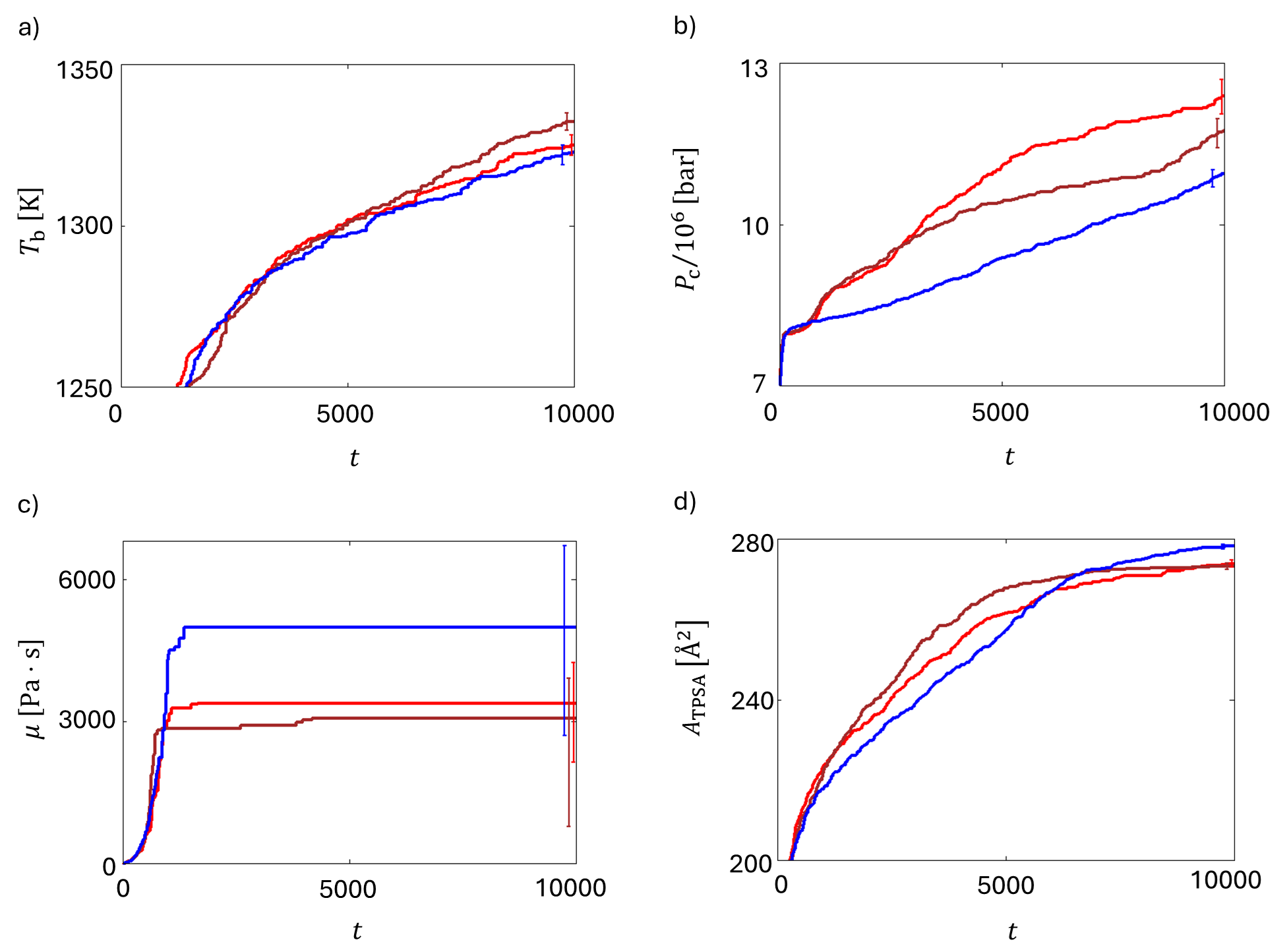}
  \caption{
  Transition plots of the cumulative maximum rewards for the demonstration problems of materials discovery.
  (a) the $T_b$ maximizing problem, (b) the $P_c$ maximizing problem, (c) the $\eta$ maximizing problem,
  (d) the $A_\text{TPSA}$ maximizing problem.
  The red, brown, and blue indicate PIUCB1, PIUCB2, and MaxSearch, respectively.
  The mean and standard errors of each $100$ independent runs are shown in (a), (b), and (d) using colored lines and error bars.
  In (c), the median values of $100$ independent runs and its $95 \%$ confidence intervals
  calculated by the bootstrap method \citep{efron1994introduction} are shown.
  The bootstrap is performed using scipy.stats.bootstrap function \citep{virtanen2020scipy} with $1,000$ resamples.
  }
  \label{material}
\end{figure}

\section{Summary}
In this paper, we generalize the proof of the order optimality of the UCB policy
as Theorem \ref{UCBTheorem} and Corollary \ref{simpleCorollary}
by introducing the oracle quantity $z_{k,t}$, which absorbs the difference of the target of bandit problems.
Using this generalized theory, we can easily confirm the order optimality
of existing UCB policies for both the total reward and max bandit problems.

Our theory also has a wide flexibility in determining the oracle quantity, $z_{k,t}$,
extending beyond the total-reward and max bandit problems.
As an example, we proposed the PIUCB algorithms,
which is order-optimal for the PI bandit problem.
These algorithms can be used for the max bandit problem in practical.
In fact, in the synthetic and demonstration max bandit problems \citep{kikkawa2024materials},
the PIUCB algorithms had comparable performances to the previously proposed MaxSearch algorithms and better than in some cases.
Considering the simplicity of the selection index of PIUCBs,
the PIUCB policies are valuable in terms of both implementation and theoretical understandings.

The flexibility of $z_{k,t}$ can lead to other UCB algorithms.
For example, higher moments of the reward distribution can be straightly used for $z_{k,t}$
and it will generates another UCB algorithm.
Because the $z_{k,t}$ is a function of the obtained rewards, 
we probably integrate the machine learning models theoretically in our framework.

Based on our theoretical foundation, it would be interesting to extend to adversarial settings,
the best-arm identification, and other applied bandit problem settings.
We hope that these studies will deepen the understanding of the bandit problems.

\acks{We thank Dr. Miyamoto in TCRDL for reviewing our work.}

\bibliography{main.bib}

\begin{thebibliography}{34}
\providecommand{\natexlab}[1]{#1}
\providecommand{\url}[1]{\texttt{#1}}
\expandafter\ifx\csname urlstyle\endcsname\relax
  \providecommand{\doi}[1]{doi: #1}\else
  \providecommand{\doi}{doi: \begingroup \urlstyle{rm}\Url}\fi

\bibitem[Achab et~al.(2017)Achab, Cl{\'e}men{\c{c}}on, Garivier, Sabourin, and Vernade]{achab2017max}
Mastane Achab, Stephan Cl{\'e}men{\c{c}}on, Aur{\'e}lien Garivier, Anne Sabourin, and Claire Vernade.
\newblock Max k-armed bandit: On the extremehunter algorithm and beyond.
\newblock In \emph{Machine Learning and Knowledge Discovery in Databases: European Conference, ECML PKDD 2017, Skopje, Macedonia, September 18--22, 2017, Proceedings, Part II 17}, pages 389--404. Springer, 2017.

\bibitem[Agrawal(1995{\natexlab{a}})]{agrawal1995continuum}
Rajeev Agrawal.
\newblock The continuum-armed bandit problem.
\newblock \emph{SIAM journal on control and optimization}, 33\penalty0 (6):\penalty0 1926--1951, 1995{\natexlab{a}}.

\bibitem[Agrawal(1995{\natexlab{b}})]{agrawal1995sample}
Rajeev Agrawal.
\newblock Sample mean based index policies by o (log n) regret for the multi-armed bandit problem.
\newblock \emph{Advances in applied probability}, 27\penalty0 (4):\penalty0 1054--1078, 1995{\natexlab{b}}.

\bibitem[Audibert and Bubeck(2010)]{audibert2010best}
Jean-Yves Audibert and S{\'e}bastien Bubeck.
\newblock Best arm identification in multi-armed bandits.
\newblock In \emph{COLT-23th Conference on learning theory-2010}, pages 13--p, 2010.

\bibitem[Auer et~al.(2002{\natexlab{a}})Auer, Cesa-Bianchi, and Fischer]{auer2002finite}
Peter Auer, Nicol{\`o} Cesa-Bianchi, and Paul Fischer.
\newblock Finite-time analysis of the multiarmed bandit problem.
\newblock \emph{Machine Learning}, 47\penalty0 (2--3):\penalty0 235--256, 2002{\natexlab{a}}.
\newblock \doi{10.1023/A:1013689704352}.

\bibitem[Auer et~al.(2002{\natexlab{b}})Auer, Cesa-Bianchi, Freund, and Schapire]{auer2002nonstochastic}
Peter Auer, Nicolo Cesa-Bianchi, Yoav Freund, and Robert~E Schapire.
\newblock The nonstochastic multiarmed bandit problem.
\newblock \emph{SIAM journal on computing}, 32\penalty0 (1):\penalty0 48--77, 2002{\natexlab{b}}.

\bibitem[Bartlett(2014)]{bartlett2014learning}
Peter Bartlett.
\newblock Learning in sequential decision problems.
\newblock \emph{Lecture Notes Stat}, 260:\penalty0 294--102, 2014.

\bibitem[Besbes et~al.(2014)Besbes, Gur, and Zeevi]{besbes2014stochastic}
Omar Besbes, Yonatan Gur, and Assaf Zeevi.
\newblock Stochastic multi-armed-bandit problem with non-stationary rewards.
\newblock \emph{Advances in neural information processing systems}, 27, 2014.

\bibitem[Browne et~al.(2012)Browne, Powley, Whitehouse, Lucas, Cowling, Rohlfshagen, Tavener, Perez, Samothrakis, and Colton]{browne2012survey}
Cameron~B Browne, Edward Powley, Daniel Whitehouse, Simon~M Lucas, Peter~I Cowling, Philipp Rohlfshagen, Stephen Tavener, Diego Perez, Spyridon Samothrakis, and Simon Colton.
\newblock A survey of monte carlo tree search methods.
\newblock \emph{IEEE Transactions on Computational Intelligence and AI in games}, 4\penalty0 (1):\penalty0 1--43, 2012.

\bibitem[Bubeck et~al.(2012)Bubeck, Cesa-Bianchi, et~al.]{bubeck2012regret}
S{\'e}bastien Bubeck, Nicolo Cesa-Bianchi, et~al.
\newblock Regret analysis of stochastic and nonstochastic multi-armed bandit problems.
\newblock \emph{Foundations and Trends{\textregistered} in Machine Learning}, 5\penalty0 (1):\penalty0 1--122, 2012.

\bibitem[Bubeck et~al.(2013)Bubeck, Cesa-Bianchi, and Lugosi]{bubeck2013bandits}
S{\'e}bastien Bubeck, Nicolo Cesa-Bianchi, and G{\'a}bor Lugosi.
\newblock Bandits with heavy tail.
\newblock \emph{IEEE Transactions on Information Theory}, 59\penalty0 (11):\penalty0 7711--7717, 2013.

\bibitem[Burnetas and Katehakis(1996)]{burnetas1996optimal}
Apostolos~N Burnetas and Michael~N Katehakis.
\newblock Optimal adaptive policies for sequential allocation problems.
\newblock \emph{Advances in Applied Mathematics}, 17\penalty0 (2):\penalty0 122--142, 1996.

\bibitem[Carpentier and Valko(2014)]{carpentier2014extreme}
Alexandra Carpentier and Michal Valko.
\newblock Extreme bandits.
\newblock \emph{Advances in Neural Information Processing Systems}, 27, 2014.

\bibitem[Cicirello and Smith(2005)]{cicirello2005max}
Vincent~A Cicirello and Stephen~F Smith.
\newblock The max k-armed bandit: A new model of exploration applied to search heuristic selection.
\newblock In \emph{The Proceedings of the Twentieth National Conference on Artificial Intelligence}, volume~3, pages 1355--1361, 2005.

\bibitem[David and Shimkin(2016)]{david2016pac}
Yahel David and Nahum Shimkin.
\newblock Pac lower bounds and efficient algorithms for the max k-armed bandit problem.
\newblock In \emph{International Conference on Machine Learning}, pages 878--887. PMLR, 2016.

\bibitem[Efron and Tibshirani(1994)]{efron1994introduction}
Bradley Efron and Robert~J Tibshirani.
\newblock \emph{An introduction to the bootstrap}.
\newblock Chapman and Hall/CRC, 1994.

\bibitem[Garivier and Moulines(2011)]{garivier2011upper}
Aur{\'e}lien Garivier and Eric Moulines.
\newblock On upper-confidence bound policies for switching bandit problems.
\newblock In \emph{International conference on algorithmic learning theory}, pages 174--188. Springer, 2011.

\bibitem[Garivier et~al.(2019)Garivier, M{\'e}nard, and Stoltz]{garivier2019explore}
Aur{\'e}lien Garivier, Pierre M{\'e}nard, and Gilles Stoltz.
\newblock Explore first, exploit next: The true shape of regret in bandit problems.
\newblock \emph{Mathematics of Operations Research}, 44\penalty0 (2):\penalty0 377--399, 2019.

\bibitem[Hopcroft(2001)]{hopcroft2001introduction}
JE~Hopcroft.
\newblock Introduction to automata theory, languages, and computation, 2001.

\bibitem[Joback and Reid(1987)]{joback1987estimation}
Kevin~G Joback and Robert~C Reid.
\newblock Estimation of pure-component properties from group-contributions.
\newblock \emph{Chemical Engineering Communications}, 57\penalty0 (1-6):\penalty0 233--243, 1987.

\bibitem[Jones et~al.(1998)Jones, Schonlau, and Welch]{jones1998efficient}
Donald~R Jones, Matthias Schonlau, and William~J Welch.
\newblock Efficient global optimization of expensive black-box functions.
\newblock \emph{Journal of Global optimization}, 13:\penalty0 455--492, 1998.

\bibitem[Kikkawa and Ohno(2024)]{kikkawa2024materials}
Nobuaki Kikkawa and Hiroshi Ohno.
\newblock Materials discovery using max k-armed bandit.
\newblock \emph{Journal of Machine Learning Research}, 25\penalty0 (100):\penalty0 1--40, 2024.

\bibitem[Kocsis and Szepesv{\'a}ri(2006)]{kocsis2006bandit}
Levente Kocsis and Csaba Szepesv{\'a}ri.
\newblock Bandit based monte-carlo planning.
\newblock In \emph{European conference on machine learning}, pages 282--293. Springer, 2006.

\bibitem[Kushner(1964)]{kushner1964new}
H.~J. Kushner.
\newblock {A New Method of Locating the Maximum Point of an Arbitrary Multipeak Curve in the Presence of Noise}.
\newblock \emph{Journal of Basic Engineering}, 86\penalty0 (1):\penalty0 97--106, 03 1964.
\newblock ISSN 0021-9223.
\newblock \doi{10.1115/1.3653121}.
\newblock URL \url{https://doi.org/10.1115/1.3653121}.

\bibitem[Lai and Robbins(1985)]{lai1985asymptotically}
Tze~Leung Lai and Herbert Robbins.
\newblock Asymptotically efficient adaptive allocation rules.
\newblock \emph{Advances in applied mathematics}, 6\penalty0 (1):\penalty0 4--22, 1985.

\bibitem[Landrum et~al.(2006)]{landrum2006rdkit}
Greg Landrum et~al.
\newblock Rdkit: Open-source cheminformatics, 2006.

\bibitem[Li et~al.(2010)Li, Chu, Langford, and Schapire]{li2010contextual}
Lihong Li, Wei Chu, John Langford, and Robert~E Schapire.
\newblock A contextual-bandit approach to personalized news article recommendation.
\newblock In \emph{Proceedings of the 19th international conference on World wide web}, pages 661--670, 2010.

\bibitem[Nishihara et~al.(2016)Nishihara, Lopez-Paz, and Bottou]{nishihara2016no}
Robert Nishihara, David Lopez-Paz, and L{\'e}on Bottou.
\newblock No regret bound for extreme bandits.
\newblock In \emph{Artificial Intelligence and Statistics}, pages 259--267. PMLR, 2016.

\bibitem[Robbins(1952)]{robbins1952some}
Herbert Robbins.
\newblock Some aspects of the sequential design of experiments.
\newblock \emph{Bulletin of the American Mathematical Society}, 58\penalty0 (5):\penalty0 527--535, 1952.
\newblock \doi{10.1090/S0002-9904-1952-09620-8}.
\newblock URL \url{https://projecteuclid.org/euclid.bams/1183517370}.

\bibitem[Srinivas et~al.(2009)Srinivas, Krause, Kakade, and Seeger]{srinivas2009gaussian}
Niranjan Srinivas, Andreas Krause, Sham~M Kakade, and Matthias Seeger.
\newblock Gaussian process optimization in the bandit setting: No regret and experimental design.
\newblock \emph{arXiv preprint arXiv:0912.3995}, 2009.

\bibitem[Streeter and Smith(2006)]{streeter2006simple}
Matthew~J Streeter and Stephen~F Smith.
\newblock A simple distribution-free approach to the max k-armed bandit problem.
\newblock In \emph{International Conference on Principles and Practice of Constraint Programming}, pages 560--574. Springer, 2006.

\bibitem[Virtanen et~al.(2020)Virtanen, Gommers, Oliphant, Haberland, Reddy, Cournapeau, Burovski, Peterson, Weckesser, Bright, et~al.]{virtanen2020scipy}
Pauli Virtanen, Ralf Gommers, Travis~E Oliphant, Matt Haberland, Tyler Reddy, David Cournapeau, Evgeni Burovski, Pearu Peterson, Warren Weckesser, Jonathan Bright, et~al.
\newblock Scipy 1.0: fundamental algorithms for scientific computing in python.
\newblock \emph{Nature methods}, 17\penalty0 (3):\penalty0 261--272, 2020.

\bibitem[Weininger(1988)]{weininger1988smiles}
David Weininger.
\newblock Smiles, a chemical language and information system. 1. introduction to methodology and encoding rules.
\newblock \emph{Journal of chemical information and computer sciences}, 28\penalty0 (1):\penalty0 31--36, 1988.

\bibitem[Yue et~al.(2012)Yue, Broder, Kleinberg, and Joachims]{yue2012k}
Yisong Yue, Josef Broder, Robert Kleinberg, and Thorsten Joachims.
\newblock The k-armed dueling bandits problem.
\newblock \emph{Journal of Computer and System Sciences}, 78\penalty0 (5):\penalty0 1538--1556, 2012.

\end{thebibliography}


\end{document}